\documentclass[a4paper]{styles/svproc}
\usepackage[numbers]{natbib}
\usepackage{multicol}
\usepackage[bookmarks=true]{hyperref}

\usepackage{amsmath} 
\usepackage{amssymb}  

\usepackage{hhline}
\usepackage{multicol}
\usepackage{multirow}
\usepackage{graphicx}
\usepackage{diagbox}
\usepackage{float}

\usepackage{booktabs}
\usepackage[table]{xcolor}

\usepackage[]{threeparttable}

\let\labelindent\relax
\usepackage{enumitem}

\newcommand{\Real}[1]{\mathbb{R}^{#1}}

\newcommand{\X}{\mathcal{X}}

\newcommand{\mK}{\mathbf{K}}

\newcommand{\Rot}{\mathbf{R}}
\newcommand{\tran}{\mathbf{t}}

\newcommand{\SE}[1]{\text{SE({#1})}}
\newcommand{\SO}[1]{\text{SO({#1})}}

\newcommand{\pa}{\mathbf{a}}
\newcommand{\pb}{\mathbf{b}}

\newcommand{\pk}{\mathbf{k}}
\newcommand{\mA}{\mathcal{A}}
\newcommand{\mB}{\mathcal{B}}

\newcommand{\inner}[1]{\langle{#1}\rangle}

\newcommand{\norm}[1]{\Vert #1 \Vert}

\newcommand{\fls}{$\mathsf{FLS}$}
\newcommand{\icp}{$\mathsf{ICP}$}
\newcommand{\ransacicp}{$\mathsf{RANSAC{-}ICP}$}
\newcommand{\flsicp}{$\mathsf{FLS{-}ICP}$}
\newcommand{\goicp}{$\mathsf{Go{-}ICP}$}
\newcommand{\teaser}{$\mathsf{TEASER}$}
\newcommand{\deepgmr}{$\mathsf{DeepGMR}$}
\newcommand{\csr}{$\mathsf{CSR}$}

\DeclareMathOperator*{\argmin}{arg\,min}

\usepackage{listings}
\usepackage{xcolor}

\definecolor{codegreen}{rgb}{0,0.6,0}
\definecolor{codegray}{rgb}{0.5,0.5,0.5}
\definecolor{codepurple}{rgb}{0.58,0,0.82}
\definecolor{backcolour}{rgb}{0.95,0.95,0.92}

\usepackage{ifthen,version}
\newboolean{include-notes}
\setboolean{include-notes}{true}
\newcommand{\tmnote}[1]{\ifthenelse{\boolean{include-notes}}%
  {\textcolor{purple}{\emph{TM says: #1}}}{}}
\newcommand{\ianote}[1]{\ifthenelse{\boolean{include-notes}}%
  {\textcolor{red}{\emph{IA says: #1}}}{}}
\newcommand{\msnote}[1]{\ifthenelse{\boolean{include-notes}}%
  {\textcolor{blue}{\emph{MS says: #1}}}{}}

\usepackage{caption}
\usepackage{subcaption}

\allowdisplaybreaks

\usepackage{url}

\begin{document}
\mainmatter              
\title{Scale-Invariant Fast Functional Registration}
\titlerunning{Scale-Invariant Fast Functional Registration}  
%
\author{Muchen Sun\inst{1} \and Allison Pinosky\inst{1} \and Ian Abraham\inst{2} \and Todd Murphey \inst{1}}
\authorrunning{Muchen Sun et al.} 
%
\tocauthor{Muchen Sun, Allison Pinosky, Ian Abraham, Todd Murphey}
\institute{Northwestern University, Evanston, IL 60208, USA,\\
\email{muchen@u.northwestern.edu},\\
\and
Yale University, New Haven, CT 06511, USA}

\maketitle              

\begin{abstract}
Functional registration algorithms represent point clouds as functions (e.g. spacial occupancy field) avoiding unreliable correspondence estimation in conventional least-squares registration algorithms. However, existing functional registration algorithms are computationally expensive. Furthermore, the capability of registration with unknown scale is necessary in tasks such as CAD model-based object localization, yet no such support exists in functional registration. In this work, we propose a scale-invariant, linear time complexity functional registration algorithm. We achieve linear time complexity through an efficient approximation of $L^2$-distance between functions using orthonormal basis functions. The use of orthonormal basis functions leads to a formulation that is compatible with least-squares registration. Benefited from the least-square formulation, we use the theory of translation-rotation-invariant measurement to decouple scale estimation and therefore achieve scale-invariant registration. We evaluate the proposed algorithm, named $\mathsf{FLS}$ (\emph{functional least-squares}), on standard 3D registration benchmarks, showing $\mathsf{FLS}$ is an order of magnitude faster than state-of-the-art functional registration algorithm without compromising accuracy and robustness. \fls{} also outperforms state-of-the-art correspondence-based least-squares registration algorithm on accuracy and robustness, with known and unknown scale. Finally, we demonstrate applying $\mathsf{FLS}$ to register point clouds with varying densities and partial overlaps, point clouds from different objects within the same category, and point clouds from real world objects with noisy RGB-D measurements. 
\end{abstract}


\section{Introduction}

Point cloud registration is the problem of transforming one set of points, through rotation, translation, and potentially scaling, in order to align with another set of points. This problem appears as a fundamental component across a variety of tasks in robotics and computer vision, such as object localization \cite{xiang2018posecnn}, medical image processing \cite{balakrishnan2019voxelmorph}, 3D reconstruction \cite{dai2017scannet}, and sensor pose estimation \cite{yokozuka2021litamin2}, etc.

\begin{figure}[!h]
    \centering
    \includegraphics[width=\textwidth]{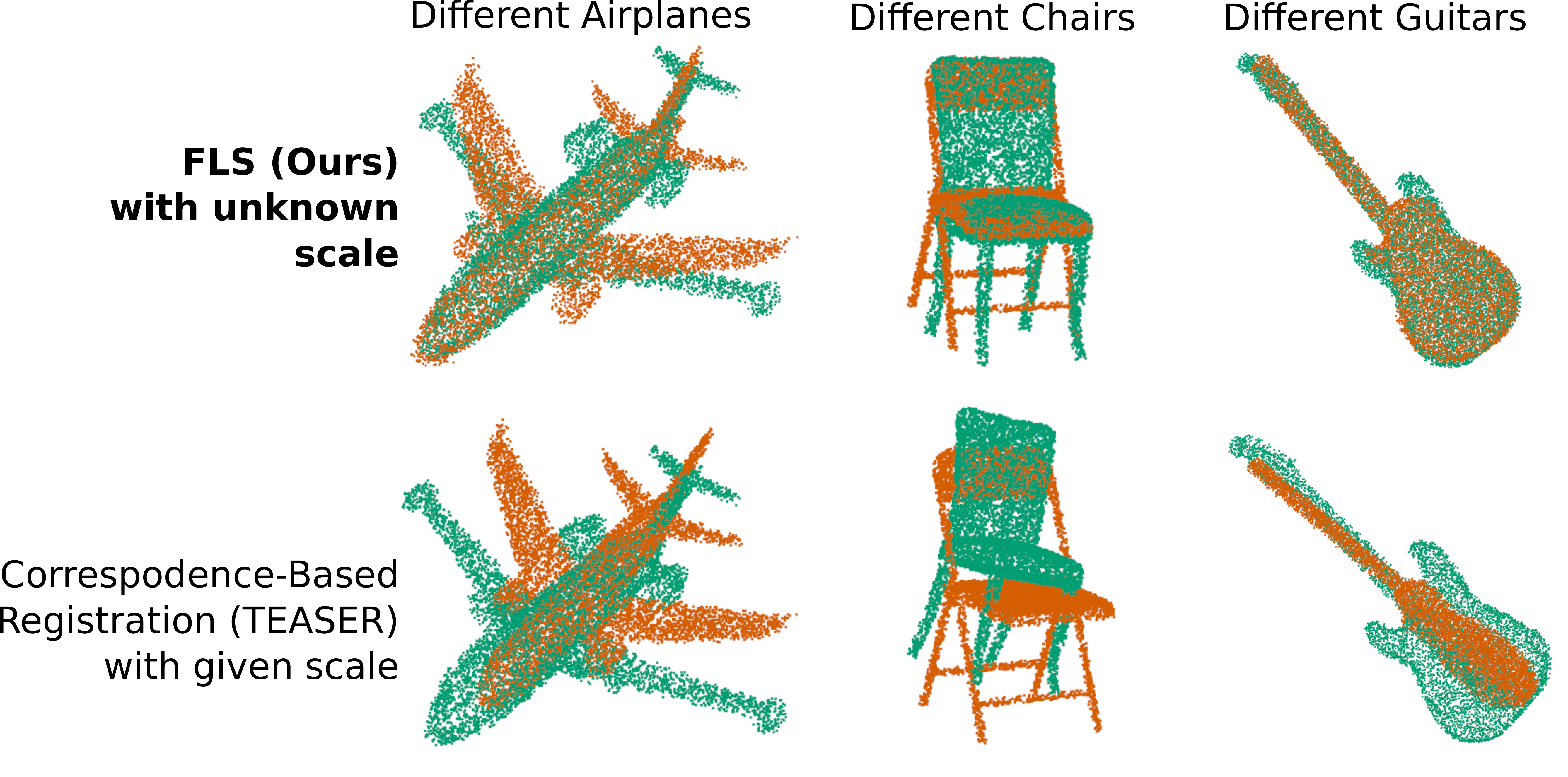}
    \vspace{-20pt}
    \caption{Functional registration does not rely on local geometry-based features but instead capturing the global structure of the point cloud. Our functional registration algorithm can register different objects' CAD models from the same category with unknown scale.}
    \label{fig: cross_objects}
    \vspace{-15pt}
\end{figure}

In its essence, the problem of registration is posed as the minimizing the \emph{distance} between two point clouds. However, point clouds are by nature unordered, and the distance metric between two unordered sets is not well defined. The most commonly used framework computes distance by introducing correspondences between the two point clouds; points that are sampled from the same location or region in the physical world. With correspondences, the unordered point sets can be ordered enabling least-squares optimization problems to be formulated that minimizes the Euclidean distances between corresponding points. The least-squares problem can be solved efficiently by methods like the Gauss-Newton method or the Levenberg-Marquardt method \cite{nocedal2006numerical}, and the computation process is parallelization-friendly. However the accuracy of correspondence estimation is sensitive to noise, and the least squares formulation is known to be vulnerable against incorrect correspondences (outliers) \cite{nocedal2006numerical}. Therefore, in practice it is hard to obtain correspondences that are sufficient for least-squares registration, especially when the point clouds are corrupted by large noise, have varying densities, or are only partially overlapped \cite{yang2020teaser, urbach2020dpdist}.

Another way to compute distance between unordered sets is to abstract the point set into a representation function (e.g., a probability distribution or an occupancy field) and minimize the distance between the representation functions---distance metric in function space is well-defined. This approach, generally  known as functional registration, avoids unreliable local geometry-based feature extraction and correspondence estimation process by capturing the global geometry of the point clouds (e.g., the general shape of the point cloud). However, functional registration algorithms are often computationally expensive, since the computing the distance between two functions is in general computationally intractable. In one of the early works of functional registration \cite{jian2010robust}, the $L^2$-distance between Gaussian-mixtures is used, which has a closed-form solution. But computing the closed-form distance has a quadratic complexity with respect to number of points, and thus is only practical for small scale problems. 

In this work, we combine the computational efficiency and simplicity of least-squares registration with the robustness of functional registration. Inspired by the study of fluid mixing \cite{mathew2005multiscale}, we propose a mapping that maps an \emph{unordered} vector set to an \emph{ordered} orthonormal basis function coefficient set. The mapped function space provides us a permutation-invariant, easy to compute distance metric between two unordered sets and formulates a least-squares registration algorithm without correspondences. This set-to-function mapping is surprisingly simple--all one needs is to average the orthonormal basis functions over the point sets. The proposed registration algorithm, named as $\mathsf{FLS}$ (functional least-squares) has the following contributions:
\begin{enumerate}
    \item $\mathsf{FLS}$ has competitive accuracy and robustness with state-of-the-art functional registration algorithm, while being an order of magnitude faster; 
    \item $\mathsf{FLS}$ has a linear time-complexity w.r.t. number of points, it can register up to 10,000 points in 1 second using only a single CPU\footnote{This is achieved through parallelization using 12 CPU threads.};
    \item $\mathsf{FLS}$ supports decoupled scale estimation, and thus can register rotation and translation with unknown scale. 
\end{enumerate}

We validate $\mathsf{FLS}$ on standard 3D registration benchmarks, showing that $\mathsf{FLS}$ is an order of magnitude faster than state-of-the-art functional registration algorithm without compromising registration accuracy and robustness (Fig.~\ref{fig: cross_objects}). We also show that $\mathsf{FLS}$ outperforms state-of-the-art correspondence-based registration algorithm while maintaining a sub-second registration speed with up to 10,000 points.


\section{Related Works}

\textbf{Correspondence-Based Registration:} The seminal work of iterative closest point (ICP) algorithm \cite{besl1992method} provides a simple yet effective way to establish correspondences. This approach has been extended and widely applied for point cloud-based odometry \cite{biber2003normal, koide2021voxelized}. Another family of methods for estimating correspondences uses feature descriptors to find geometrically distinctive points and match them across point clouds, such as the fast point feature histogram (FPFH) \cite{rusu2009fast} and the normal aligned radial feature (NARF) \cite{steder2011point}. However, with the presence of noise, the local geometry of the point clouds can be blurred and feature descriptors may return wrong correspondences (outliers). 

Besides better correspondence estimation, performance of correspondence-based registration can also be improved through a more robust least-squares solver. Random sampling consensus (RANSAC) \cite{fischler1981random} is one of the most commonly used framework for robust least-squares registration, and its result is often used as the initial estimation for ICP algorithms. However, RANSAC has an exponential runtime w.r.t. the outlier ratio and may still fail with large outlier ratio \cite{yang2020teaser}. Other robust registration algorithms include the Guaranteed Outlier REmovel (GORE) algorithm \cite{bustos2017guaranteed} and the practical maximum clique (PMC) algorithm \cite{bustos2019practical}. Yang et al. propose the truncated least-squares (TLS) formulation and provide first certifiable robust registration algorithm in \cite{yang2020teaser}, which could handle registration problems with outlier ratio higher than $90\%$. The authors also introduce the theory of translation-rotation-invariant measurements (TRIMs) and translation-invariant-measurements (TIMs) for decoupled scale estimation (with unknown rotation and translation) and decoupled rotation estimation (with known scale but unknown translation), given correspondences between the original point sets.

Another approach to estimate correspondences relies on deep learning pipelines. Deep learning methods are first applied to design data-driven feature descriptors, such as in \cite{thomas2019kpconv, tinchev2019skd}. Learning-based methods are also utilized for combined feature extraction and correspondence matching, such as in \cite{zeng20173dmatch, gojcic2019perfect}. \cite{yuan2020deepgmr} propose a probabilistic registration method that extracts pose-invariant correspondences through probabilistic point-to-distribution matching inside a learned latent space. However, deep learning methods are inherently limited to the training set and may not generalize to new objects.

\textbf{Functional Registration:} As its name indicates, functional registration algorithms represent point clouds as functions, such as spatial density distribution or Gaussian-mixture distribution, and perform registration through minimization of the ``difference'' between the functions. Commonly used metrics include the $L^2$-distance in function space or the Kullback-Leibler divergence. In the seminal work \cite{1544863, jian2010robust}, a Gaussian-mixture model is constructed by attaching a Gaussian distribution to each point in a point cloud, and the authors compute the closed form solution for the $L^2$-distance between such Gaussian-mixture distributions. However, this method has a quadratic time-complexity with respect to the number of points and is thus limited to small-scale problems. More recently, Bernreiter et al. \cite{bernreiter2021phaser} introduce a correspondence free and global registration algorithm by correlating the point clouds in the frequency domain through Fourier transform. However, the proposed method is not in a continuous domain, thus the accuracy depends on the resolution of Fourier transform. It also requires extra information such as intensity or range. The work of nonparametric continuous sensor registration \cite{JMLR:v22:20-1468} is the closest to our work. The authors show that by constructing representation functions from the point clouds in a reproducing kernel Hilbert space, minimizing the $L^2$-distance between the functions is equivalent to maximizing the inner product of the functions. However, computing the inner product,  though it can be accelerated through the sparsity of the kernel, still has a worst-case quadratic time complexity. This approach is also specified for visual odometry by integrating extra information such as color, intensity, or semantic information in \cite{Grizzle-RSS-19, unified_cvo}. However, none of the aforementioned algorithm supports registration with unknown scale.


\section{Notation}

In the point cloud registration problem, we are given two point sets $\mA = \{\pa_i\}_{i=1}^{N_\mA}$ and $\mB = \{\pb_j\}_{j=1}^{N_\mB}$, with $\pa_i, \pb_j \in \Real{d}, d\in\{2, 3\}$. We assume there are no redundant elements in the point sets. In this paper, we look for three transformations in order to register point set $\mA$ to $\mB$: scale, rotation, and translation, which are represented by a positive real number $s\in\Real{+}$, a rotation matrix $\Rot\in\SO{d}$, and a $d$-dimensional vector $\tran\in\Real{d}$, respectively. We use the notation $\mA(\Rot, \tran, s)=\{s \Rot \pa_i + \tran\}_{i=1}^{\mathcal{N}_\mathcal{A}}$ to represent the new set after applying the above transformation to every point in set $\mA$. We call point set $\mA$ the \emph{source point set/cloud}, and point set $\mB$ the \emph{target point set/cloud}.

\section{Functional Point Cloud Registration}

\subsection{The Delta-Distance Between Unordered Sets}
For any subset of $\mA\subseteq\Real{d}$, we can attach a Dirac delta function to each of its element, and thus transform a set to a function. We name this function the \emph{delta-mixture function} of a set:
\begin{definition}[Delta-mixture function]
    Given a finite set $\mA{=}\{\pa_1  \dots \pa_{N_\mA}\}$, $\pa_i\in\Real{d}$, the delta-mixture function $\delta_{\mA} : \Real{d} \mapsto \{0, +\infty\}$ is defined as:
    \begin{align}
        \delta_{\mA}(x) & = \frac{1}{N_\mA} \sum_{i}^{N_{\mA}} \delta_{\pa_i}(x) \label{eq: delta_mixture}
    \end{align} where $\delta_{\pa_i}(x)$ is the Dirac delta function with infinite impulse at $\pa_i$.
\end{definition}

The delta-mixture function (\ref{eq: delta_mixture}) can be considered as an occupancy field, as it fully represents the occupancy of the point cloud in the space. This motivates us to use the $L^2$-distance between two delta-mixture functions as the distance between two unordered point sets. Furthermore, delta-mixture functions allow closed-form orthonormal decomposition which can simplify the computation of $L^2$-distance (see Section \ref{subsec: delta_decomposition}), unlike Gaussian-mixtures.

\begin{definition}[Delta-distance between sets] \label{def: exact_delta_distance}
The delta-distance between two sets is defined as the $L^2$-distance between the two delta-mixture functions that are generated from the sets:
\begin{align}
    & d_\delta(\mA, \mB)^2 = \inner{\delta_{\mA}(x){-}\delta_{\mB}(x), \delta_{\mA}(x){-}\delta_{\mB}(x)} = \int_{\X} \big(\delta_{\mA}(x){-}\delta_{\mB}(x)\big)^2 dx \\
    & = \int_{\X} \left( \frac{1}{N_\mA}\sum_{i}^{N_\mA} \delta_{\pa_i}(x) - \frac{1}{N_\mB}\sum_{i}^{N_\mB} \delta_{\pb_i}(x) \right)^2 dx \label{eq: exact_delta_distance}
\end{align}     
\end{definition}

\begin{lemma}
The delta-distance between two sets $\mA$ and $\mB$ is permutation-invariant (if sets $\mA$ and $\mB$ are ordered).
\label{lemma: permutation_invariant}
\end{lemma}
\begin{proof}
The permutation-invariance of delta-distance can be proved from the permutation-invariance of the summation operator in Equation (\ref{eq: exact_delta_distance}). 
\end{proof}

The $L^2$-distance is symmetric and positive-definite, which makes it an ideal metric to measure geometric similarity. However, this formulation in equation (\ref{eq: exact_delta_distance}) cannot be directly used in practice since the delta-distance is zero only when two point sets are \emph{exactly} the same. Below we propose a continuous relaxation of the delta-distance through orthonormal decomposition of delta-mixture function. This continuous relaxation takes an advantage of the Dirac delta function---the inner product between a Dirac delta function and a continuous function has a closed form solution. This relaxation also allows points to be not exactly at the same location.

\subsection{Continuous Relaxation of the Delta-Distance} \label{subsec: delta_decomposition}

\begin{definition}[Orthonormal basis function]
An infinite set of functions $\{f_1, f_2, \dots\}$ forms the orthonormal bases of the function space with the same domain and co-domain, if the following two properties are satisfied:
\begin{align}
    \inner{f_i, f_j} & = \int_{\X} f_i(x) f_j(x) dx = 0, \forall i\neq j \textsf{ (Orthogonality)} \label{eq: orthogonality} \\
    \inner{f_i, f_i} & = \int_{\X} f_i(x)^2 dx = 1, \forall i \textsf{ (Normality)} \label{eq: normality}
\end{align} where $\inner{\cdot,\cdot}$ denotes the inner product in function space.
\end{definition}

\begin{corollary}[Orthonormal decomposition]
Given a set of orthonormal function bases $\{f_1, f_2, \dots\}$, a function $g(x)$ with the same domain and co-domain can be decomposed as \cite{mallat1993matching}:
\begin{align}
    g(x) = \sum_{k=1}^{\infty} \inner{f_k, g} f_k(x) = \sum_{k=1}^{\infty} \int f_k(x) g(x) dx \cdot f_k(x) .
\end{align} \label{corollary: orthonormal_decomposition}
\end{corollary}

\begin{corollary} \label{corollary: closed_form_inner_product}
The inner product between a Dirac delta function $\delta_{\pa}(x)$ and a continuous function $g(x)$ has a closed-form solution:
\begin{align}
    \inner{\delta_{\pa}, g} = g(\pa) .
\end{align}
\end{corollary}
\begin{proof}
This can be proved using the definition of delta function: $\int \delta_\pa(x) dx = 1$.
\end{proof}

\begin{lemma} \label{lemma: orthonormal_decomposition_delta_mixture}
Based on Corollary \ref{corollary: orthonormal_decomposition} and \ref{corollary: closed_form_inner_product}, given a set of orthonormal function bases $\{f_1, f_2, \dots\}$, a delta-mixture function $\delta_{\mA}(x) = \frac{1}{N_\mA} \sum_{i}^{N_\mA} \delta_{\pa_i}(x)$ can be decomposed as:
\begin{align}
    \delta_{\mA}(x) = \sum_{k=1}^{\infty} c_k^\mA \cdot f_k(x),\quad  c_k^\mA = \frac{1}{N_\mA} \sum_{i}^{N_\mA} f_k(\pa_i) \label{eq: decomposed_delta_function} .
\end{align}
\end{lemma}

\begin{remark}
The delta-mixture function (\ref{eq: delta_mixture}) can be considered as an occupancy field extracted from the point cloud, and the decomposed delta-mixture function (\ref{eq: decomposed_delta_function}) can be considered as an implicit representation of the field, similar to other implicit field representations such as the Neural Radiance Field (NeRF)\cite{NEURIPS2020_55053683}. In practice we use normalized Fourier basis functions as the orthonormal bases, and NeRF shares a similar strategy by using Fourier feature mapping (also called positional encoding). The difference is NeRF does not use Fourier features as orthonormal bases, but instead use them to transform the Neural Tangent Kernel (NTK) of the multi-layer perceptron (MLP) to a stationary kernel \cite{NEURIPS2020_55053683}.
\end{remark}

\begin{theorem}
The delta-distance between two sets (\ref{eq: exact_delta_distance}) can be computed as:
\begin{align}
    d_\delta(\mA, \mB)^2 = \sum_{k=1}^{\infty} \left(  \frac{1}{N_\mA} \sum_{i=1}^{N_\mA} f_k(\pa_i) {-} \frac{1}{N_\mB} \sum_{i=1}^{N_\mB} f_k(\pb_i) \right)^2 = \sum_{k=1}^{\infty} \left(c_k^\mA - c_k^\mB\right)^2 \label{eq: orthonormal_decomposition_l2_distance} .
\end{align} where $\{f_1, f_2, \dots\}$ is the set of orthonormal function bases.
\label{theorem: orthonormal_decomposition_l2_distance}
\end{theorem}
\begin{proof}
See appendix A.
\end{proof}

\begin{remark}
The structure of equation (\ref{eq: orthonormal_decomposition_l2_distance}) exhibits a similar pattern as the PointNet \cite{Aoki_2019_CVPR}, in the sense that they both use summation to achieve permutation-invariance. The difference is that PointNet has one sum pooling function, whereas our formulation has multiple summations and we show that when the number of summations approaches infinity, equation (\ref{eq: orthonormal_decomposition_l2_distance}) approaches the exact $L^2$-distance between two delta-mixture functions.
\end{remark}

With Theorem \ref{theorem: orthonormal_decomposition_l2_distance}, we can approximate a continuous delta-distance between two sets by choosing \emph{a finite set of continuous orthonormal function bases} to compute (\ref{eq: orthonormal_decomposition_l2_distance}). Below we specify the details for using this approximation for point cloud registration.

\subsection{Functional Point Cloud Registration} \label{sec: functional_point_cloud_registration}

In this paper we choose to use normalized Fourier basis functions for orthonormal decomposition.

\begin{definition}[Normalized Fourier basis functions] Define a function $f:\X\mapsto\Real{}$, where $\X = [L_1^l, L_1^u]\times\cdots\times[L_d^l, L_d^u] \subset\Real{d}$ is a $d-$dimensional rectangular space, $L_i^l$ and $L_i^u$ are the lower and upper bound for the $d$-th dimension, respectively. For any function defined as above, we can define the following normalized Fourier basis functions as the bases for orthonormal decomposition~\cite{mathew2005multiscale}:
\begin{align}
    f_\pk(x) & = \frac{1}{h_\pk} \prod_{i=1}^{d} \cos\left(\bar{k}_i (x_i - L_i^l)\right) \label{eq: fourier_basis}
\end{align} where
\begin{align}
    x & = (x_1, x_2, \dots, x_N) \in \Real{d} , \quad \pk = [k_1, \cdots, k_d] \in [0, 1, 2, \cdots, \mK]^d \subset \mathbb{N}^{d} \nonumber \\
    \bar{k}_i & = \frac{k_i \pi}{L_i^u-L_i^l}, \quad h_\pk = \left( \prod_{i=1}^{d} \frac{L_i^u - L_i^l}{2} \right)^{\frac{1}{2}} \nonumber
\end{align}
\end{definition}

Now, we finally combine (\ref{eq: orthonormal_decomposition_l2_distance}) and (\ref{eq: fourier_basis}) to define the least-squares formulation for functional registration, named as \emph{functional least-squares} ($\mathsf{FLS}$). We first introduce $\mathsf{FLS}$ for rotation and translation estimation with given scale.

\begin{definition}[$\mathsf{FLS}$ registration with known scale]
When the scale $s$ is given, $\mathsf{FLS}$ for rotation and translation estimation is defined as:
\begin{align}
    \argmin_{\Rot,\tran} \sum_{\pk\in[0, \dots, \mK]^d} r_\pk^2(\Rot, \tran) \label{eq: fls_T_obj}
\end{align} where 
\begin{align}
    r_\pk(\Rot, \tran) & = \sqrt{\lambda_\pk} \cdot \Big( \inner{\delta_{\mA(\Rot, \tran, s)}, f_\pk} - \inner{\delta_{\mB}, f_\pk} \Big) \\
    & = \sqrt{\lambda_\pk} \cdot \left( \frac{1}{N_\mA} \sum_{i=1}^{N_\mA} f_\pk(s\cdot\Rot\pa_i+\tran_i) - \frac{1}{N_\mB} \sum_{i=1}^{N_\mB} f_\pk(\pb_i) \right) \label{eq: fls_T_residual} \\
    \lambda_\pk & = \left(1 + \norm{\pk}^2\right)^{-\frac{d+1}{2}}
\end{align} where $\{\lambda_\pk\}$ is a convergent series to bound the $L^2$-distance. 
\end{definition}

This specific form of $\lambda_\pk$ above is inspired by the Sobolev norm formulation in \cite{mathew2005multiscale}, which assigns diminishing weights to basis functions with increasing frequencies. It encourages the algorithm to capture the global structure of the point cloud, and thus more robust against outlier points.

Note that the complexity of evaluating the $\mathsf{FLS}$ residual function (\ref{eq: fls_T_residual}) has a \emph{linear complexity} of number of points. Further, even though in (\ref{eq: fls_T_obj}) the number of residuals is $\mK^d$, in practice we found $\mK=5$ to be sufficient, thus for a 3D registration task, only $125$ residual functions are needed. Furthermore, since (\ref{eq: fls_T_obj}) is a multi-scale measure \cite{mathew2005multiscale}, the number of residuals \emph{does not} depend on the size, scale or density of the point cloud, and thus can be assumed fixed. \\

\noindent \textbf{Decoupled Scale Estimation. } The concept of \emph{translation and rotation invariant measurements} ($\mathsf{TRIMs}$) is proposed in \cite{yang2020teaser} in order to decouple scale estimation from rotation and translation estimation.

\begin{definition}[Translation rotation invariant measurements]
The translation rotation invariant measurements ($\mathsf{TRIMs}$) of a point set $\mA=\{\pa_1, \dots, \pa_{N_\mA}\}$ is a set of scalar measurements:
\begin{align}
    \mathsf{TRIM}_{\mA} = \Big\{ \norm{\pa_i-\pa_j} \quad \Big\vert \quad 1\leq i,j \leq N_\mA, i\neq j \Big\} 
\end{align}
\end{definition}

The intuition behind $\mathsf{TRIMs}$ is that, the distances between one point cloud's own points are only affected by the scale of the point cloud, and two point clouds with same number of points and same scale should have same $\mathsf{TRIMs}$. Based on this concept, the decoupled scale estimation of $\mathsf{FLS}$ is formulated as minimizing the delta-distance between the $\mathsf{TRIMs}$ of two point cloud.

\begin{definition}[$\mathsf{FLS}$ for scale estimation]
Given two point sets with \textbf{unknown} relative rotation and translation, $\mathsf{FLS}$ for scale estimation is defined as the following one-dimensional registration problem:
\begin{align}
    \argmin_{s} \sum_{\pk\in[0, \dots, \mK]} r_\pk^2(s) \label{eq: fls_s_obj}
\end{align} where 
\begin{align}
    r_\pk(s) & = \sqrt{\lambda_\pk} \cdot d_{\delta} (\mathsf{TRIM}_{\mA(s)}, \mathsf{TRIM}_{\mB})^{\frac{1}{2}} \\
    & = 2 \sqrt{\lambda_\pk} \cdot \Bigg( \frac{\sum_{i=1}^{N_\mA} \sum_{j=i+1}^{N_\mA} f_\pk(s\norm{\pa_i-\pa_j})}{(N_\mA-1)N_\mA} - \frac{\sum_{i=1}^{N_\mB} \sum_{j=i+1}^{N_\mB} f_\pk(\norm{\pb_i-\pb_j})}{(N_\mB-1)N_\mB} \Bigg) \label{eq: fls_residual} \\
    \lambda_\pk & = \left(1 + \norm{\pk}^2\right)^{-1}
\end{align}
\end{definition}


\section{Experiments and Application}

\subsection{Implementation Details}

We implement $\mathsf{FLS}$ with the Ceres nonlinear least-squares solver \cite{ceres-solver}. We use its implementation of the Levenberg-Marquardt method to solve the least-squares in (\ref{eq: fls_T_obj}). When estimating the optimal rotation and translation, we optimize directly in the space of $\SE{3}$. Through all the following experiments, we use 5 orthonormal basis functions per dimension (thus 125 residuals for 3D point cloud registration, and 5 residuals for scale estimation). 

Given the fast registration speed of \fls{}, we can use its registration result to initialize \icp{} to further refine the result with almost no speed loss. We name this method as \flsicp{}. We release the source code of $\mathsf{FLS}$ and benchmark tools at \url{https://sites.google.com/view/fls-isrr2022/}.

\subsection{Rationale for Test Algorithms}

We compare \fls{} and \flsicp{} with six other algorithms: iterative closest point algorithm (\icp)\cite{besl1992method}, iterative closest point algorithm with initial estimation from the random sample consensus algorithm ($\mathsf{RANSAC}$) and FPFH features ($\mathsf{RANSAC{-}ICP}$), global iterative closest point algorithm (\goicp)\cite{yang2015go}, truncated least squares estimation and semidefinite relaxation (\teaser)\cite{yang2020teaser}, continuous sensor registration (\csr)\cite{JMLR:v22:20-1468}, and deep Gaussian mixture registration (\deepgmr)\cite{yuan2020deepgmr}. 

We choose \icp{} and \ransacicp{} as the baselines, and \goicp{} provides global optimization for classical \icp{}. We choose \teaser{} as the state-of-the-art robust correspondence-based registration algorithm, we use the FPFH features \cite{rusu2009fast} to estimate correspondences for \teaser{} through out the experiments\footnote{It's unclear how FPFH features can be applied across different scales, and thus for scale estimation, we use the default order the point clouds as correspondences.}. \csr{} is the state-of-the-art functional registration. \deepgmr{} is a deep learning-based pipeline that provides globally optimal registration without iterative optimization. We accelerate \fls{}, \csr{}, and \teaser{} all with 12 CPU threads using OpenMP.

\subsection{Rationale for Experiment Design}

We test algorithms exclusively on the ModelNet40 dataset\cite{Wu_2015_CVPR}, which contains CAD models of 40 categories of objects. These CAD models allow generation of varying types of synthetic point clouds for testing with a given ground truth. This dataset has been widely used as a benchmark for point cloud registration algorithms. The benefit of synthetic dataset is that it provides a systematic and controllable way to examine different aspects of an algorithm. We examine the tested algorithms from the aspects of: (1) robustness against noise; (2) robustness against partial overlap and varying densities; (3) time efficiency; (4) sensitivity to initialization. We scale all the test point clouds into a unit cube before testing to better illustrate the value of evaluation metrics. Furthermore, for every algorithm tested, we align the geometrical centers of the source and target point clouds before registration. This does not eliminate the necessity of translation estimation, but will improve the performance of local optimization-based algorithms, such as \fls{}, \csr{}, and \icp{}. For our algorithm, we define a cube with each dimension being [-1, 1] as the rectangular space in (\ref{eq: fourier_basis}) to bound both the source and target point clouds. 

Some categories in the ModelNet40 dataset contain objects that have more than one correct registration results, for example, any rotation around the center axis of a water bottle can be considered as correct. Thus, we choose ten categories of objects with no obvious symmetric geometry as the test data: \textit{airplane, bed, chair, desk, guitar, mantel, monitor, piano, sofa, stairs}. Furthermore, since \deepgmr{} is trained on the ModelNet40 dataset, we select the first five objects from the test set of the ten selected categories, which create a library of 50 objects for our experiments. In addition, we demonstrate using our algorithm for real world object localization from noisy RGB-D point cloud.

We use five metrics to evaluate the performance of each algorithm tested: (1) rotation error; (2) translation error; (3) running time; (4) failure rate; (5) exact recovery rate. A registration result is considered failed if the rotation error is larger than 45 degrees or translation error is larger than 0.5 meters, such as when the algorithm gets stuck at a poor local minimum or has completely wrong correspondences. We exclude these results when computing the mean and standard deviation of rotation and translation error. A registration result is considered an exact recovery if the rotation error is smaller than 5 degrees and the translation error is smaller than 0.03 meter.

\subsection{Robustness Against Noise}

\begin{figure}[!h]
    \vspace{-5pt}
    \centering    
    \includegraphics[width=\textwidth]{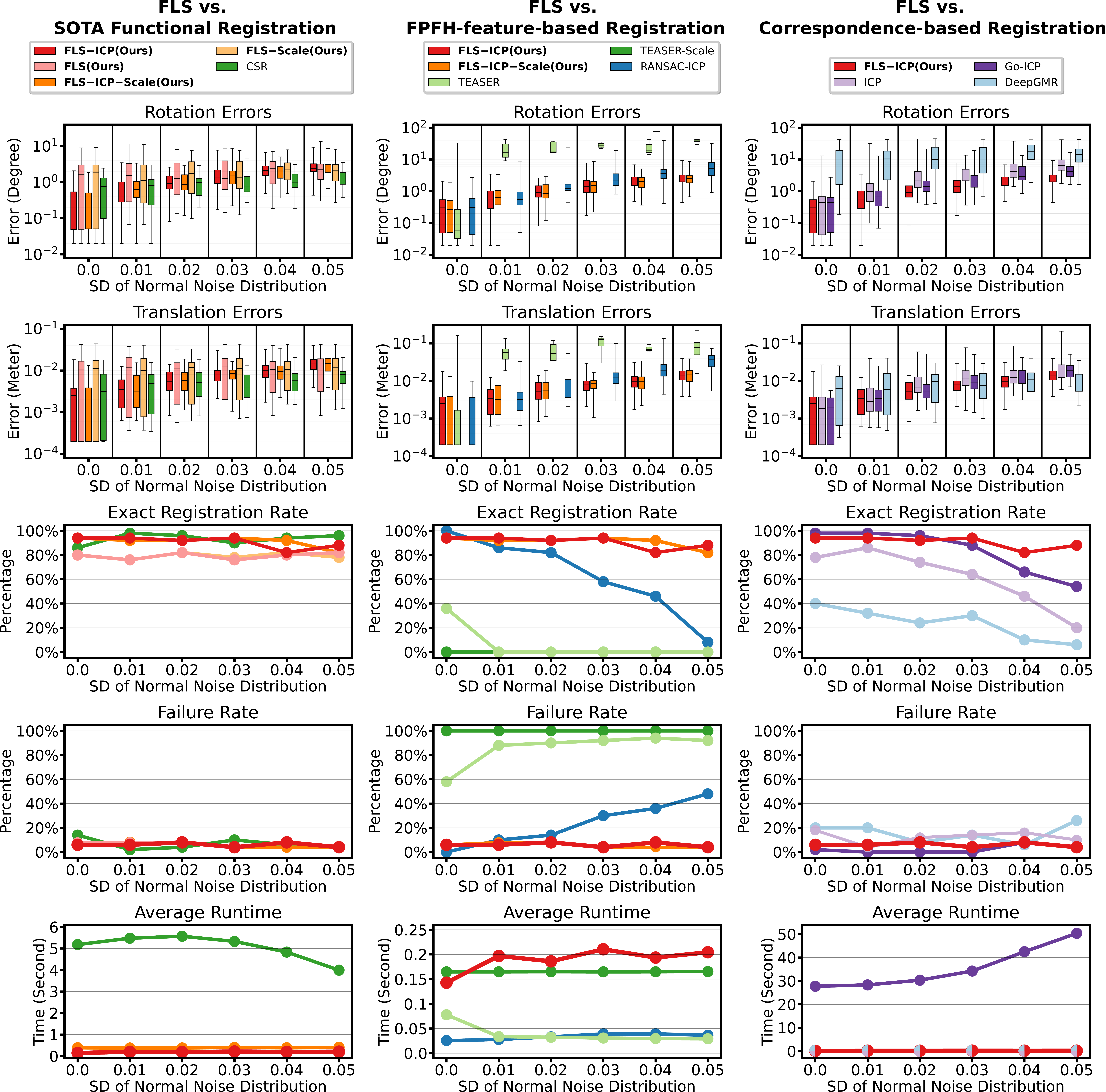}
    \vspace{-15pt}
    \caption{\textit{\textbf{Left}}: Our algorithms reaches a similar level of accuracy and robustness as \csr{}, but is an order of magnitude faster and achieves real-time registration speed ($< 0.1s$ per registration). \textit{\textbf{Middle}}: \fls{} outperforms FPFH feature-based registration algorithms with superior robustness against noise. $\mathsf{TEASER{-}Scale}$ has 100\% failure rate due to lack of sufficient correspondences. \textit{\textbf{Right}}: \fls{} outperforms other correspondence-based registration algorithms on both robustness against noise and computation efficiency. The runtime curves of some algorithms overlap due to the high runtime of \goicp. (Note: failed registration results are omitted from all rotation and translation plots.)}
    \label{fig: noise_test}
    \vspace{-2em}
\end{figure}

For each object, we generate a point cloud with 1024 points from its CAD model\footnote{We use the pcl\_obj2ply tool from the PCL library~\cite{Rusu_ICRA2011_PCL}.}, and then scale the point cloud to a unit cube. For ground-truth rotation, we first uniformly sample a vector $\mathbf{u}=[u_1, u_2, u_3]\in\Real{3}_{0^+}$, then we uniformly sample an angle $\theta$ from $[-\frac{\pi}{2}, \frac{\pi}{2}]$ to generate a random rotation matrix $\Rot = \exp([\mathbf{u}\theta]_{\times})$. We uniformly sample the ground-truth translations from $[1, 2]$ for each axis\footnote{Due to the alignment of geometrical centers of source and target point clouds in pre-processing, all the tested algorithms estimate the translation within a unit cube.}. Lastly, we randomly shuffle the source cloud before each test. We add noise sampled from zero-mean normal distributions with standard deviation varying from 0.01 to 0.05. For all algorithms we initialize the estimation as an identity matrix (no rotation and translation).

To test the scale-invariance of our algorithm, we randomly sample scale factor between 2 to 5 and apply our algorithm and \teaser{} (labeled as $\mathsf{FLS{-}Scale}$, $\mathsf{FLS{-}ICP{-}Scale}$, and $\mathsf{TEASER{-}Scale}$, respectively) to recover rotation and translation with unknown scale. For algorithms tested we initialize the scale estimation as 1. The experimental results are shown in Fig.~\ref{fig: noise_test}. Note that, even though $\mathsf{FLS{-}ICP{-}Scale}$ and $\mathsf{TEASER{-}Scale}$ are tested with unknown scale, we present the results alongside tests with known scale for comparison purposes.

\subsection{Robustness Against Partial Overlap And Varying Densities}

\begin{figure}[!h]
    \vspace{-25pt}
    \centering
    \includegraphics[width=0.95\textwidth]{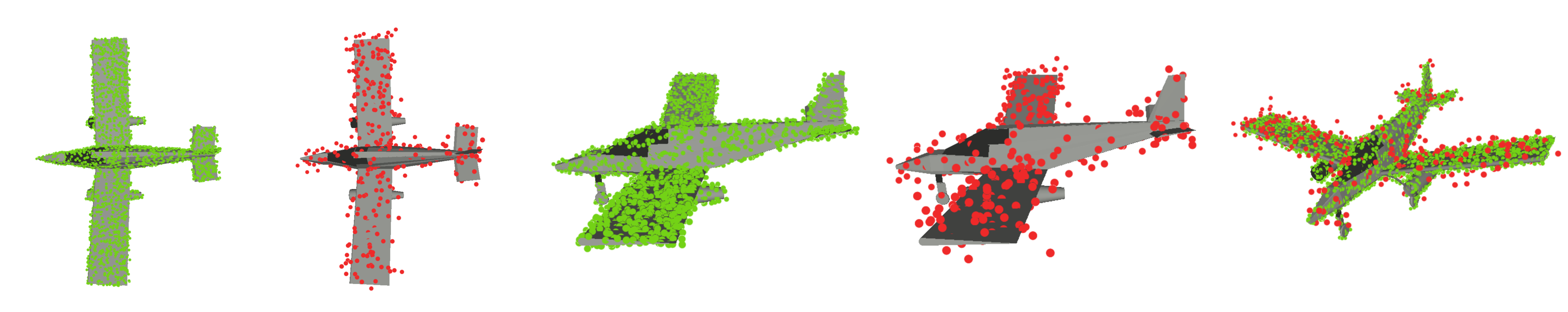}
    \vspace{-10pt}
    \caption{Example of a synthetic point cloud measurement (red) and the dense noiseless point cloud (green) generated from the CAD model.}
    \label{fig: synthetic_measurement}
    \vspace{-15pt}
\end{figure}

To evaluate the performance of the tested algorithms on point clouds with partial overlap and varying densities, we design the task of registering an object from CAD model to noisy real world measurement. For every object tested in the last experiment on robustness against noise, we first generate a uniformly covered, dense, noiseless point cloud with 4096 points from the CAD model, then we simulate a ``real world'' point cloud measurement of the same object by synthesising a sparse point cloud from only three views. This synthesized point cloud measurement only contains 512 points, and we add normally distributed noise with zero mean and standard deviation of $0.02$. We choose this noise level to match the commonly observed real world noise from commercial RGB-D cameras. Furthermore, since the synthesized point cloud measurement is generated from only three views, it will not cover the whole surface of the object and the covered areas will have varying densities. An example of the synthetic point cloud measurement is shown in Fig. \ref{fig: synthetic_measurement}. Similar to the last experiment on robustness against noise, we also generate random scale from $[2, 5]$ to test the scale-invariance of our algorithm. The experiment results are shown in Table \ref{table: partial_overlap}.

\begin{table*}
\centering
\vspace{-10pt}
\begin{tabular*}{\textwidth}{c | @{\extracolsep{\fill}} ccccc}
\hline
        & \begin{tabular}[c]{@{}l@{}}Rotation\\Error(deg.)\end{tabular} & \begin{tabular}[c]{@{}l@{}}Translation\\Error(meter)\end{tabular} & \begin{tabular}[c]{@{}l@{}}Runtime\\(second)\end{tabular} & \begin{tabular}[c]{@{}l@{}}Exact\\Recovery\\Rate\end{tabular} & \begin{tabular}[c]{@{}l@{}}Failure\\Rate\end{tabular}  \\ 
\hline
$\mathsf{FLS}$ & 9.22${\pm}$9.33 & 0.09${\pm}$0.07 & 0.16${\pm}$0.05 & 8\% & 10\% \\ 
$\mathsf{FLS{-}ICP}$ & 2.57${\pm}$3.03 & 0.04${\pm}$0.06 & 0.17${\pm}$0.05 & 56\% & 8\% \\ 
$\mathsf{FLS{-}Scale}$ & 9.56${\pm}$9.04 & 0.10${\pm}$0.08 & 2.88${\pm}$0.68 & 4\% & 10\% \\ 
$\mathsf{FLS{-}ICP{-}Scale}$ & 2.35${\pm}$1.42 & 0.06${\pm}$0.07 & 2.89${\pm}$0.68 & 32\% & 8\% \\ 
$\mathsf{CSR}$ & \bf{2.08${\pm}$1.14} & 0.04${\pm}$0.07 & 8.80${\pm}$16.82 & 38\% & 10\% \\ 
$\mathsf{DeepGMR}$ & 22.02${\pm}$12.76 & 0.16${\pm}$0.12 & \bf{0.007${\pm}$0.002} & 2\% & 62\% \\ 
$\mathsf{TEASER}$ & 43.13${\pm}$1.69 & 0.26${\pm}$0.03 & 0.25${\pm}$0.08 & 0\% & 96\% \\ 
$\mathsf{ICP}$ & 3.32${\pm}$3.75 & 0.06${\pm}$0.09 & 0.012${\pm}$0.002 & 50\% & 6\% \\
$\mathsf{Go{-}ICP}$ & 2.20${\pm}$1.80 & 0.04${\pm}$0.07 & 64.44${\pm}$28.43 & \bf{58\%} & \bf{4\%} \\
$\mathsf{RANSAC{-}ICP}$ & 2.12${\pm}$0.88 & \bf{0.03${\pm}$0.03} & 0.09${\pm}$0.03 & 28\% & 54\% \\
\hline
\end{tabular*}
\caption{When registering point clouds with partial overlap and varying densities, $\mathsf{FLS{-}ICP}$ shares the highest recovery rate with the globally optimized \goicp, while being \emph{two} orders of magnitudes faster. Our algorithm's rotation and translation error is also close to the top. The running time of \fls{} increases with unknown scale because of the high computation complexity of generating TRIMs.  (Note: failed registration results are omitted from all rotation and translation results. We do not report $\mathsf{TEAER{-}Scale}$ due to its 100\% failure rate.)}
\label{table: partial_overlap}
\vspace{-20pt}
\end{table*}

\begin{figure}[!h]
    \centering
    \includegraphics[width=\textwidth]{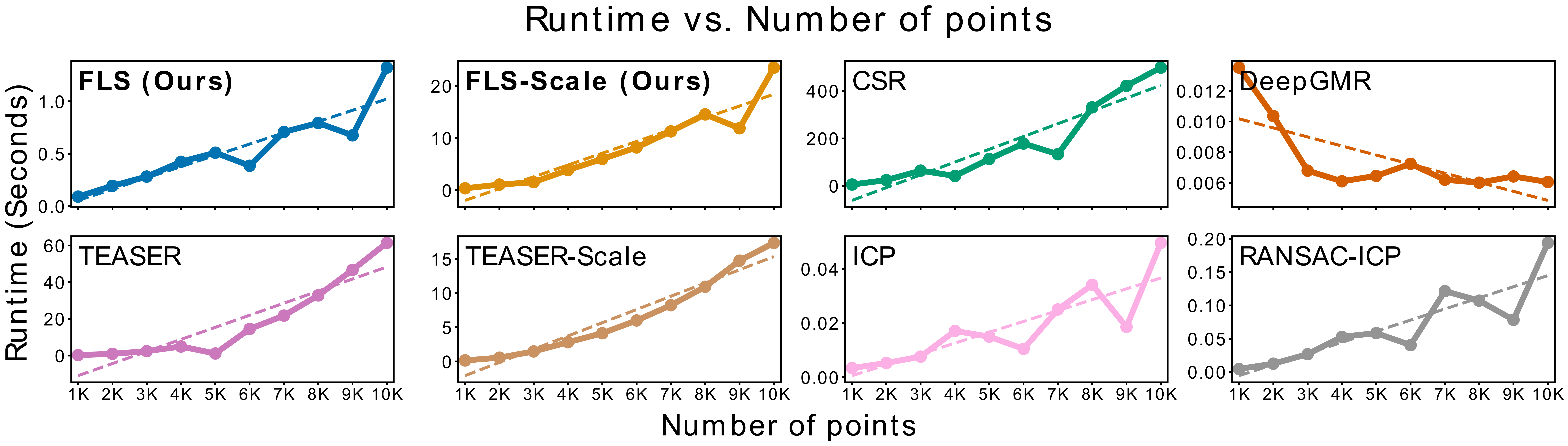}
    \vspace{-20pt}
    \caption{Experiment results on time efficiency, we fit dashed straight line to every curve. \fls{} (with known scale) exhibits a linear time complexity. \fls{} with unknown scale exhibits a superlinear time complexity due to the high computation cost of generating TRIMs, but it is still faster than \csr{} and \teaser{} (with known and unknown scale).}
    \label{fig: time_efficiency}
    \vspace{-2em}
\end{figure}

\subsection{Time Efficiency}

To test the time efficiency of the algorithms, we generate point clouds with numbers of points varying from 1000 to 10000 and record the registration time. The results are reported in Fig.~\ref{fig: time_efficiency}. We did not test time efficiency with \goicp{} due to the high computation time of the Branch-and-Bound scheme, which is already reflected in the last two experiments.

\subsection{Sensitivity to Initialization}

The \fls{} objectives (\ref{eq: fls_T_obj}) and (\ref{eq: fls_s_obj}) are not convex, in practice we found the scale estimation can always converge to a good minimum, but it is not the case for rotation and translation registration. To better understand how the influence from the non-convexity of the objective (\ref{eq: fls_T_obj}), we generate ground-truth rotation from rotation angles varying from $\frac{\pi}{6}$ to $\pi$ and initialize all algorithms with an identity matrix (no rotation and translation). We compare our algorithm with other local registration algorithms: \csr{} and \icp{}. We report each algorithm's exact recovery rate, failure rate and average rotation error in Fig.~\ref{fig: initialization_sensitivity}.

\begin{figure}[!h]
    \vspace{-5pt}
    \centering
    \includegraphics[width=\textwidth]{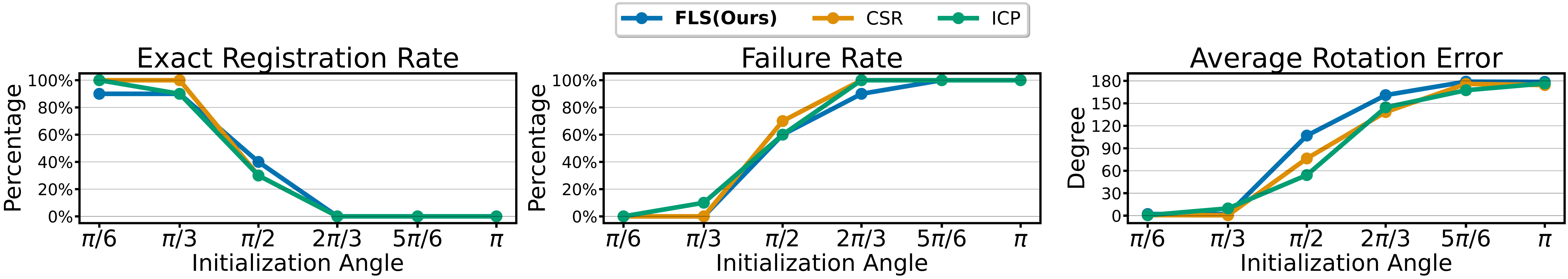}
    \vspace{-15pt}
    \caption{Experiment results on sensitivity to initialization. All three local registration algorithms tested share similar sensitivities to initialization.}
    \label{fig: initialization_sensitivity}
    \vspace{-5pt}
\end{figure}

\begin{figure}[!h]
    \centering
    \includegraphics[width=\textwidth]{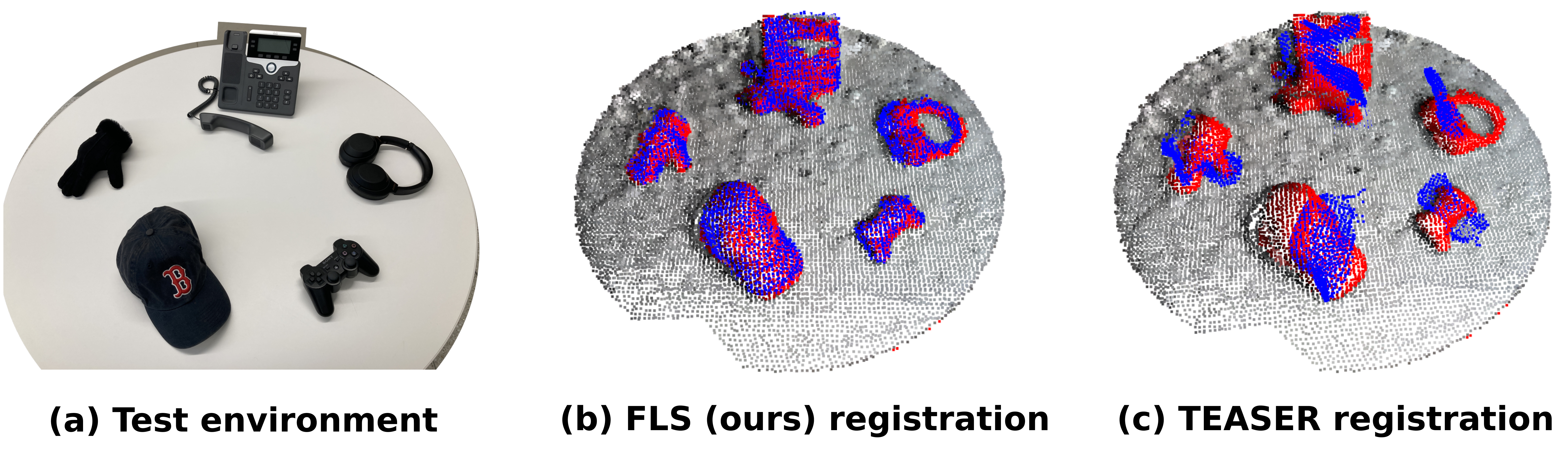}
    \vspace{-2em}
    \caption{$\mathsf{FLS}$ is able to register all objects (hat, controller, headphones, phone, glove) in the test environment, while \teaser{} exhibits large rotation error. \emph{This experiment further demonstrates the benefits of the correspondence-free nature of our algorithm.}}
    \label{fig: real_world_test}
    \vspace{-15pt}
\end{figure}

\subsection{Real World Application}

We use a RealSense D435 camera and the ROS package $\mathsf{rtabmap\_ros}$ to generate real-world test point clouds. We generate two point clouds of the same test environment, as shown in Fig.~\ref{fig: real_world_test}(a). First, we translate and rotate the second cloud away from the first cloud. Then, we crop the point clouds to each object using the RGB data. This process generates five objects for registration from each test environment point cloud. In Fig.~\ref{fig: real_world_test} (b)(c), the objects from the first point cloud are shown in red and those from the second point cloud are shown in blue. We register each object individually and show all of the objects aligned on the original (not color-filtered) point cloud. Note that the point clouds are generated from different motions of the camera, thus there are no exact point-to-point correspondences. The results are shown in Fig.~\ref{fig: real_world_test}.

\vspace{-5pt}


\section{Conclusions}

In this work, we propose a new scale-invariant functional registration algorithm with a new distance metric between unordered point sets (delta-distance). We show that our algorithm has linear complexity with respect to number of points and can be an order of magnitude faster than state-of-the-art functional registration algorithm, without compromising accuracy and robustness. Our algorithm also outperforms state-of-the-art correspondence-based algorithms on accuracy and robustness even with unknown scale and partial overlap. In the future, we will look into improving the robustness of the delta-distance and explore using the delta-distance as the loss function in learning-based pipelines. \\

{\footnotesize \noindent\textbf{Acknowledgements. } This material is supported by NSF Grant CNS-1837515, ONR Grant N00014-21-1-2706, and HRI Grant HRI-001479. Any opinions, findings and conclusions or recommendations expressed in this material are those of the authors and do not necessarily reflect the views of the aforementioned institutions.}



%
%

\bibliographystyle{IEEEtran}
{\footnotesize
\bibliography{ref}}

\newpage

\begin{appendix} 
\section{Proof} \label{appendix: proof}

\noindent\textbf{Proof for Theorem \ref{theorem: orthonormal_decomposition_l2_distance}}
\begin{proof}
Based on Definition \ref{def: exact_delta_distance} and Lemma \ref{lemma: orthonormal_decomposition_delta_mixture}, the  delta-distance can be written as:
\begin{align}
    & d_\delta(\mA, \mB)^2 \\
    & = \inner{\delta_{\mA}(x){-}\delta_{\mB}(x), \delta_{\mA}(x){-}\delta_{\mB}(x)} \\
    & = \int_{\X} \big(\delta_{\mA}(x){-}\delta_{\mB}(x)\big)^2 dx \\
    & = \int_{\X} \left( \frac{1}{N_\mA}\sum_{i}^{N_\mA} \delta_{\pa_i}(x) - \frac{1}{N_\mB}\sum_{i}^{N_\mB} \delta_{\pb_i}(x) \right)^2 dx \\
    & = \int_{\X} \left( \sum_{k=1}^{\infty} c_k^\mA \cdot f_k(x) - \sum_{k=1}^{\infty} c_k^\mB \cdot f_k(x) \right)^2 dx \\
    & = \int_{\X} \left( \sum_{k=1}^{\infty} (c_k^\mA - c_k^\mB) \cdot f_k(x) \right)^2 dx \\
    & = \int_{\X} \left[ \sum_{k_1=0}^{\infty} \sum_{k_2=0}^{\infty} (c_{k_1}^\mA-c_{k_1}^\mB) (c_{k_2}^\mA-c_{k_2}^\mB) f_{k_1}(x) f_{k_2}(x) \right] dx \\
    & = \sum_{k_1=0}^{\infty} \sum_{k_2=0}^{\infty} \left[ (c_{k_1}^\mA-c_{k_1}^\mB) (c_{k_2}^\mA-c_{k_2}^\mB) \left( \int_{\X} f_{k_1}(x) f_{k_2}(x) dx \right) \right] \\
    & = \sum_{k_1=0}^{\infty} \sum_{k_2=0}^{\infty} \left[ (c_{k_1}^\mA-c_{k_1}^\mB) (c_{k_2}^\mA-c_{k_2}^\mB) \inner{f_{k_1}, f_{k_2}} \right] 
\end{align}
By using the orthogonality and normality of the basis functions (\ref{eq: orthogonality}, \ref{eq: normality}), the double summation above can be simplified as a single summation:
\begin{align}
    & \sum_{k_1=0}^{\infty} \sum_{k_2=0}^{\infty} \left[ (c_{k_1}^\mA-c_{k_1}^\mB) (c_{k_2}^\mA-c_{k_2}^\mB) \inner{f_{k_1}, f_{k_2}} \right] \\
    & = \sum_{k=0}^{\infty} \left[ (c_{k}^\mA-c_{k}^\mB) (c_{k}^\mA-c_{k}^\mB) \inner{f_{k}, f_{k}} \right] \\
    & = \sum_{k=0}^{\infty} (c_{k}^\mA-c_{k}^\mB)^2
\end{align} which completes the proof.
\end{proof}

\end{appendix}

\end{document}